\documentclass{article}
\usepackage{spconf,amsmath,graphicx}
\usepackage[hidelinks]{hyperref}

\usepackage{textpos}
\usepackage{xcolor}
\usepackage{tabularx,longtable,multirow,caption}
\usepackage{url} 
\usepackage[english]{babel}
\usepackage{tikz}
\usetikzlibrary{shapes.geometric}
\usetikzlibrary{arrows.meta} 
\usetikzlibrary{calc,matrix,positioning}
\usepackage{amsthm}
\usepackage{amsmath}
\usepackage{graphicx}
\usepackage{subcaption}
\usepackage{dsfont}
\usepackage{epstopdf} 
\usepackage{array}
\usepackage{latexsym}

\usepackage{booktabs}

\allowdisplaybreaks

\theoremstyle{plain}
\newtheorem{theorem}{Theorem}[section]
\newtheorem{definition}[theorem]{Definition}
\newtheorem{lemma}[theorem]{Lemma}

\newtheorem{remark}[theorem]{Remark}
\newtheorem{corollary}[theorem]{Corollary}

\newcommand{\lstc}{l_{\text{BTC}}}

\newcommand{\R}[0]{\mathds{R}} 
\newcommand{\N}[0]{\mathds{N}} 
\newcommand{\sns}{\sigma_{\text{n}}^2}

\newcommand{\bx}{\boldsymbol{x}}
\newcommand{\boldf}{\boldsymbol{f}}

\newcommand{\Kff}{K_{\boldsymbol{f},\boldsymbol{f}}}

\newcommand{\MN}{BTC}

\title{Accelerated training of Gaussian Processes using banded square exponential covariances}
\name{Emily C. Ehrhardt and   Felipe Tobar}
\address{Imperial College London}
\begin{document}
\ninept
\maketitle
\begin{abstract}
We propose a novel approach to computationally efficient GP training based on the observation that square-exponential (SE) covariance matrices contain several off-diagonal entries extremely close to zero. We construct a principled procedure to eliminate those entries to produce a \emph{banded}-matrix approximation to the original covariance, whose inverse and determinant can be computed at a reduced computational cost, thus contributing to an efficient approximation to the likelihood function. We provide a theoretical analysis of the proposed method to preserve the structure of the original covariance in the 1D setting with SE kernel, and validate its computational efficiency against the variational free energy approach to sparse GPs. 

\end{abstract}
\section{Introduction}
\label{sec:intro}

\subsection{Gaussian process and sparse approximations}
Gaussian processes (GPs) are non-parametric probabilistic models for time series specified by their mean $m$ and covariance $k$, whereby for a finite set of inputs $\bx \in \R^n$ the corresponding function values satisfy $f(\bx) \sim \mathcal{N}(m(\bx),k(\bx,\bx))$. In the Bayesian regression setting, denoting the observations by $\mathcal{D}= \{(x_i,y_i)\}_{i=1}^n$, GPs are trained by minimising the negative log-likelihood (NLL)
\begin{equation}
\label{eq:NLL}
     l_{\text{NLL}}
 =   \frac{1}{2}\boldsymbol{y}^\top (\Kff + \sns I )^{-1} \boldsymbol{y} + \frac{1}{2} \log|\Kff + \sns I| + \frac{n}{2} \log(2\pi),
\end{equation}
where we assumed a Gaussian likelihood $\boldsymbol{y} = f(\bx) + \varepsilon$ with $\varepsilon \sim \mathcal{N}(\boldsymbol{0},\sns I)$, $\boldsymbol{y}$  and $\boldf$ are stacked notations for the observations and latent function values, respectively, and $K_{\boldf,\boldf} = k(\bx,\bx)$.

For test inputs $\bx_\ast$, the predictive posterior of $\boldf_\ast = f(\bx_\ast)$ is also Gaussian $p(\boldf_\ast | \boldsymbol{y},\bx,\bx_\ast) = \mathcal{N}(\boldsymbol{\mu}_\ast, \Sigma_\ast)$ with
\begin{align}
    \boldsymbol{\mu}_\ast &= K_{\ast,\boldsymbol{f}}\left(K_{\boldsymbol{f},\boldsymbol{f}} + \sns I \right)^{-1} \boldsymbol{y}, \\  
    \Sigma_\ast &= K_{\ast,\ast} - K_{\ast,\boldsymbol{f}} \left(K_{\boldsymbol{f},\boldsymbol{f}} + \sns I \right)^{-1}K_{\boldsymbol{f},\ast},
\end{align}
where $K_{\ast,\ast} = k(\bx_\ast,\bx_\ast)$ and $K_{\ast,\boldf}^\top = K_{\boldf,\ast} = k(\bx,\bx_\ast)$. Both training and inference thus require solving linear systems involving $\Kff + \sns I$, which has complexity $\mathcal{O}(n^3)$.

Sparse GP methods \cite[Ch.~8]{rasmussen2005gaussian}, \cite{overview_sparseGP:quinonero-candela05a} address the computational bottleneck by introducing $m \ll n$ \emph{inducing variables} that summarise the data and thus allow for a low-rank approximation of the Gram matrix $\Kff$. These approximations reduce the computational complexity to $\mathcal{O}(nm^2)$ in general, and thus make GPs applicable to datasets with several thousands of observations. Some approaches, such as Fully Independent Training Conditional (FITC) \cite{FITC}, perform exact inference under an approximate model, while others, such as Variational Free Energy (VFE) \cite{titsias09a}, use variational inference.

Standard sparse GP methods such as VFE and FITC perform well when a few inducing points are sufficient to capture the overall structure of the latent function. However, long time-series typically require more inducing points to cover the input space, which adds substantial computational cost, often offsetting the efficiency gains. This motivates approaches that focus on local correlations, exploiting the fact that distant points are often weakly dependent, to achieve efficient and accurate modelling of long time-series.


\begin{figure}[t]
    \centering
     \begin{subfigure}[t]{0.49\linewidth}
    \vspace{0mm}
    \centering
        \includegraphics[scale=0.295]{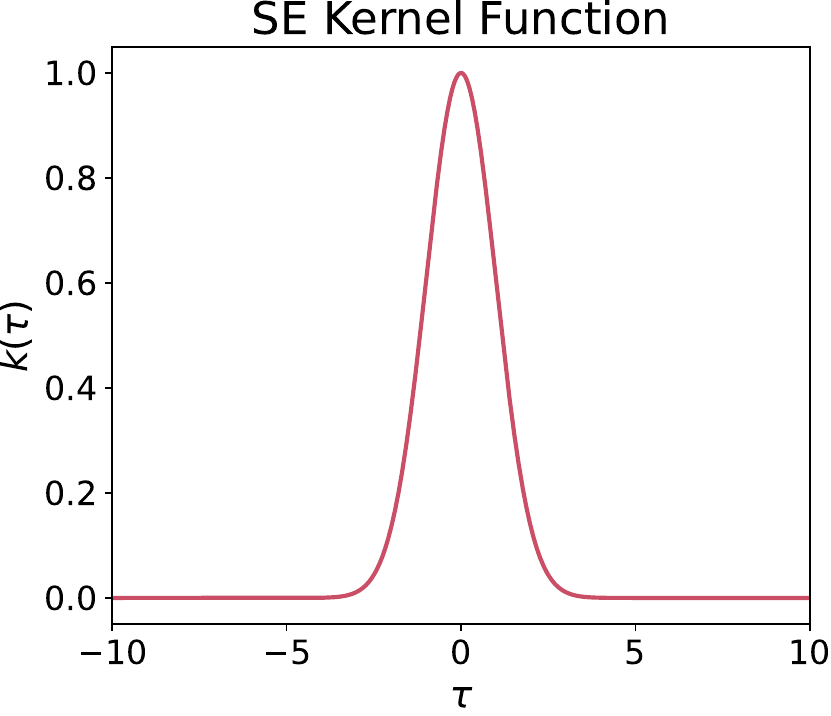}
    \end{subfigure}
    \hfill
    \begin{subfigure}[t]{0.49\linewidth}
     \vspace{0mm}
    \centering
        \includegraphics[scale=0.295]{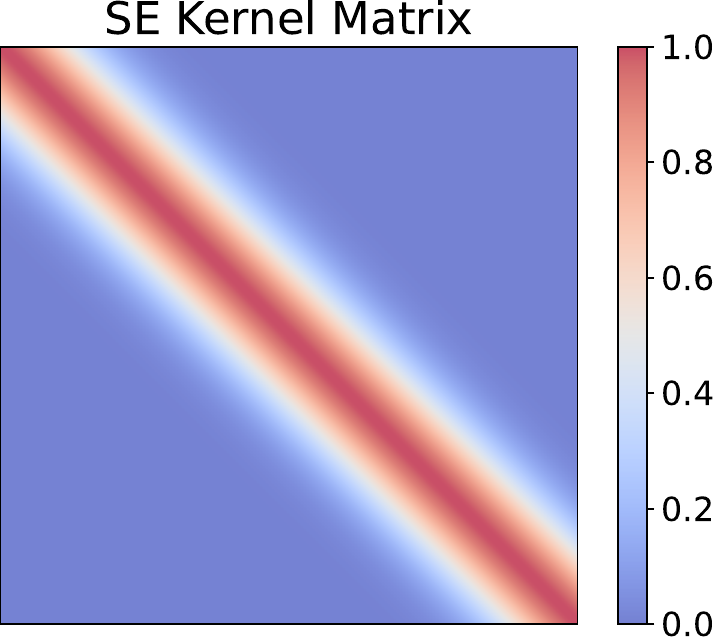}
    \end{subfigure}
    \caption{Square exponential kernel $k(\tau) = \sigma^2 \exp \left( - \tau^2/2\ell^2 \right)$, with  $\sigma = \ell = 1$. \textbf{Left:} function plotted for $\tau = |x- y |$. \textbf{Right:} Gram matrix for 1000 equidistant points with $x_{i+1}-x_i = 0.01$. }
    \label{fig:SE_kernel}
\end{figure}

\subsection{Exploiting the \emph{almost sparse} covariance}

We focus on GPs supported on $\R$ using the square exponential (SE) kernel, where, as illustrated in Fig.~\ref{fig:SE_kernel}, a vast subset of the pairwise covariances decay exponentially to zero. We exploit this \emph{approximate} sparsity by setting some of its off-diagonals to zero after a defined bandwidth $k$. Because of the sparsity pattern of $k$-banded matrices, the solution of its related linear system and log-determinant can be computed at a cost $\mathcal{O}(nk^2)$ by using the LU decomposition  \cite[Prop. 2.3]{AppliedNumericalLinearAlgebra} or the Cholesky decomposition \cite[Ch. 4.3.5]{MatrixComputations}. Therefore, we perform approximate-likelihood GP training by replacing the covariance matrix with its banded approximation. This method, termed banded covariance training (BTC), yields the same computational cost as the celebrated sparse GP approaches \cite{FITC, titsias09a}, without introducing (variational) parameters that need to be optimised. 

The defined setup (1-dimensional, SE kernel, Gaussian noise) allows us to identify a bound for the bandwidth $k$ that provides maximum computational gain while preserving i) the positive definiteness of the approximated Gram matrix, and ii) the validity of the posterior predictive distribution. The proposed BTC approach is tested against FITC, VFE and standard ML in terms of accuracy and computational complexity on synthetic and real-world datasets of varying size. Critically, we show experimentally that BTC outperforms all considered benchmarks in the defined setup.

\section{Banded Training Covariance (BTC)}
\label{sec:1D}



\subsection{Truncation of covariance matrices}
\label{sec:truncation}

We define $k$-banded matrices and the operator to construct them.

\begin{definition}{($k$-banded Matrix)}
    A matrix $A = (A_{ij}) \in \R^{n \times n}$ is called banded with  bandwidth $k$ if 
    $$
    A_{ij} = 0  \text{ for }  |i-j| > k.
    $$
\end{definition}

\begin{definition}{(Cut-Off Operator)}
\label{def:cutoff_op}
    For $n,k \in \N$ with $k \leq n$ we define the linear ``cut-off" operator $L_k \colon \R^{n \times n} \to \R^{n \times n}$ by 
    $$
    \left( L_k(A) \right)_{ij} = \begin{cases}
        A_{ij} & \text{if } |i-j| \leq k, \\
        0 & \text{if not}
    \end{cases}
    $$
    for $A = (A_{ij}) \in \R^{n \times n}$ and $1 \leq i,j \leq n$.
\end{definition}

\begin{remark}
    The cut-off operator is invariant to the addition of a noise covariance term, that is,
$$
L_k(A + \sns I) = L_k(A) + \sns I,\quad  A \in \R^{n \times n},\, k \leq n.
$$ 
\end{remark}

Building on Def.~\ref{def:cutoff_op}, we propose an approximate GP training procedure by replacing the covariance matrix $\Kff$ in eq.~\eqref{eq:NLL} with its $k$-banded version $L_k(\Kff)$. That is, 
\begin{align}
\label{eq:BTC_loss}
     \lstc
     & =   \frac{1}{2}\boldsymbol{y}^\top (L_k(\Kff) + \sns I )^{-1} \boldsymbol{y} + \frac{n}{2} \log(2\pi)\nonumber\\
     &\quad + \frac{1}{2} \log(|L_k(\Kff) + \sns I|) .
\end{align}
\begin{remark}
    The computational cost to evaluate the proposed loss $\lstc$ in eq.~\eqref{eq:BTC_loss} is $\mathcal{O}(k^2n)$.
\end{remark}
After training, the predictive distribution can be obtained from the \textit{approximate joint prior distribution}, i.e., the joint prior using the $k$-banded version of the Gram matrix given by:
\begin{equation}
\label{eq:appr_prior}
    q(\boldsymbol{y},\boldsymbol{f}_\ast) = \mathcal{N} \left(\boldsymbol{0}, \begin{bmatrix}
        L_k(K_{\boldsymbol{f},\boldsymbol{f}}) + \sns I & K_{\boldsymbol{f}, \ast}\\
        K_{\ast, \boldsymbol{f}} & K_{\ast,\ast}
    \end{bmatrix}  \right).
\end{equation}

Eq.~\eqref{eq:appr_prior} implies that, $\boldsymbol{y}$ is Gaussian with covariance matrix $L_k(K_{\boldsymbol{f},\boldsymbol{f}}) + \sns I$, which should be symmetric positive definite matrix. Additionally, to ensure a valid posterior distribution, the posterior covariance matrix also needs to be symmetric and positive definite. We study these conditions in Secs.~\ref{sec:StructurePreservation} and \ref{sec:validity} respectively.

\subsection{Related methods}
BTC is most closely related to covariance tapering \cite{Furrer01092006}, which introduces sparsity by replacing the covariance function $k$ with a tapered version
\begin{equation*}
    k_{\text{tap}}(\bx, \bx^\prime) = \Tilde{k}(\bx, \bx^\prime) k(\bx, \bx^\prime),
\end{equation*}
where $\Tilde{k}$ is a compactly supported covariance function. Here, the resulting Gram matrix $k_{\text{tap}}$ is positive definite as the Hadamard product of two positive definite matrices. BTC would coincide with covariance tapering for $\Tilde{k}(\bx_i,\bx_j) =  \mathds{1}(|i-j| \leq k)$. However, this function is generally not positive definite and thus not a valid choice for $\Tilde{k}$ in the framework of covariance tapering. This phenomenon is avoided in our setting by choosing the optimal $k$ for the SE kernel case as described in Sec.~\ref{sec:StructurePreservation}. Another related line of work introduces sparsity into the precision matrix, the inverse Gram matrix, using a conditional independence assumption \cite{vecchia}, \cite{nngp},\cite{general_vecchia}. 

Despite their seemingly equivalent computational complexity, the relationship between BTC and sparse GP methods using inducing variables is not straightforward. When approximating $\Kff$ by $L_k(\Kff)$, BTC implicitly assumes that each observation is only correlated to the $k$ previous and the $k$ following observations, meaning that, beyond $2k$ time indices, samples are considered to be uncorrelated. This is in contrast with the global conditional independence assumed by sparse GPs given the inducing points.

 \section{Theoretical results}

\subsection{Structure preservation of the Cut-Off operator}
\label{sec:StructurePreservation}

To guarantee that the marginal prior $\boldsymbol{y} \sim \mathcal{N}(0,L_k(K_{\boldsymbol{f},\boldsymbol{f}}) + \sns I)$ in Eq.~\eqref{eq:appr_prior} is valid (i.e.\ non-degenerate), the bandwidth $k$ should be chosen such that $L_k(K_{\boldsymbol{f},\boldsymbol{f}}) + \sns I$ is symmetric and positive definite. In addition to theoretical soundness, this also promotes practical numerical stability and is even necessary to apply the Cholesky decomposition.  Since the Gram matrix of the SE kernel $K_{\boldsymbol{f},\boldsymbol{f}}$ is known to be symmetric and positive definite \cite[Thm.~6.10]{Wendland_2004}, we characterise the choice of $k$ to preserve positive definiteness through the following lemma.

\begin{lemma}{(Positive definiteness of a cut-off matrix)} \label{Lem:general_pd}
    Consider a symmetric positive definite matrix $A = (A_{ij}) \in \R^{n \times n}$. Choose $k \leq n$ such that
    $$
    \max_{i = 1, \dots, n} \sum_{|i-j|> k} |A_{ij}|  < \lambda_{\min}(A).
    $$
    Then, the cut-off matrix $L_k(A)$ is positive definite.
\end{lemma}

\begin{proof}
    Since $A$ and $A-L_k(A)$ are symmetric, we have 
    \begin{align*}
    \lambda_{\min}(A) &= \min_{\bx \neq 0} \frac{\bx^\top A \bx}{\|\bx\|_2^2} \\
    \lambda_{\max}(A-L_{k}(A)) &= \max_{\bx \neq 0} \frac{\bx^\top (A-L_k(A))\bx}{\|\bx\|_2^2}.
    \end{align*}
    Hence, for $\bx \in \R^n$ the following inequality holds 
    \begin{align*}
    \bx^\top L_k(A)\bx &= \bx^\top A \bx - \bx^\top(A-L_k(A)) \bx \\
    &\geq \left[\lambda_{\min}(A) - \lambda_{\max}(A-L_k(A)) \right] \|\bx\|^2.
    \end{align*}
    Following Gershgorin's Circle Theorem \cite[Thm. 6.1.1]{Horn_Johnson_2012}, we have
    \begin{align*}
    \lambda_{\max}(A- L_k(A)) &\leq \max_{i = 1, \dots, n} \sum_{\substack{j=1 \\ j \neq i} }^n |(A- L_k(A))_{ij}| \\
    &= \max_{i = 1, \dots, n} \sum_{|i-j|> k} |A_{ij}|. 
    \end{align*}
    This concludes the proof. 
\end{proof}


Lemma \ref{Lem:general_pd} depends on the size of the Gram matrix and thus on the number of observations $n$. In practice, $k$ has to be chosen based on the kernel values, since the number of observations does not determine the dynamic features of the underlying latent function. Critically, in cases with long observation times, and thus large $n$, we expect the appropriate bandwidth $k$ not to grow with $n$,  while preserving positive definiteness. To this end, the following main result considers Gram matrices of the SE kernel with a noise term.

\begin{theorem}(Positive definiteness of a cut-off SE-kernel matrix with noise) \label{Thm:SE-KErnel_pd} 
Consider a set of points $x_1 < \dots < x_n$, the SE kernel $K_\theta$, parameters $\theta = (\sigma^2, \ell^2) \in \R^2_{++}$ and $\sns > 0$. 
Assume that $k \in \N$ is chosen such that
$$
(K_\theta)_{ij} = \sigma^2 \exp \left( - \tfrac{(x_i-x_j)^2}{2\ell^2}\right) \leq \varepsilon \leq \tfrac{\sns  3 \delta^2}{4 \ell^2} \exp \left( - \tfrac{3  \delta^2}{2 \ell^2} \right),
$$
for all $|i-j| > k$, where $\delta = \min_{i \neq j} |x_i-x_j|>0$.
Then we have that 
$$
L_k \left( K_\theta + \sns  I \right) \succ 0.
$$
\end{theorem}

\begin{proof}
Since $K_\theta$ is positive semidefinite, all of its eigenvalues are non-negative and hence $\lambda_{\min} (K_\theta + \sns I) \geq \sns $. Consider $|i-j|>k$. First, we assume $j > i$. Then,
\begin{align*}
    (K_\theta)_{ij} 
    &= \sigma^2 \exp \left( - \frac{((x_i- x_{i+k+1}) + (x_{i+k+1}-x_j))^2}{2\ell^2}\right) \\
    &\leq \varepsilon \exp \left( - \frac{ 3(|j-i|-k-1) \delta^2}{2\ell^2}\right),
\end{align*}
where we used that $|x_i - x_{i+k+1}| \geq \delta$ and $|x_{i+k+1} - x_j| \geq (j-i-k-1) \delta$. By the symmetry of $K_\theta$ this also holds for $i>j$.\\
Hence, for $i=1,\dots,n$ we have
\begin{align*}
    \sum_{|i-j|> k} |(K_\theta + \sns I)_{ij}| &=  \sum_{|i-j|> k}  \sigma^2 e^{\frac{-(x_i-x_j)^2}{2\ell^2}} < 2\varepsilon \sum_{k=0}^\infty  e^{ \frac{ -3k\delta^2}{2\ell^2}}.
\end{align*}

Applying a result from the integral test for convergence \cite[Thm. 9.2.6]{IntroductionToRealAnalysis} yields
\begin{align}
\label{eq:pd_SE_kernel}
\nonumber
    \max_{i = 1, \dots, n} \sum_{|i-j|> k} |(K_\theta + \sns I)_{ij}| &<  2\varepsilon \int_{-1}^\infty  \exp \left( - \frac{ 3k\delta^2}{2\ell^2}\right) \: dk\\
    \nonumber
    &=  \varepsilon \frac{4 \ell^2}{3 \delta^2} \exp \left( \frac{ 3\delta^2}{2\ell^2}\right)\leq \sns ,
\end{align}
where we used the definition of $\varepsilon$ in the last line. Applying Lemma \ref{Lem:general_pd} yields the result.
\end{proof}

Assuming that $
    \delta := \min_{\substack{i \neq j}} |x_i-x_j|
    $
    remains invariant with increasing $n$, meaning that more observations span an increasing temporal range, Thm. \ref{Thm:SE-KErnel_pd} only depends on the kernel parameters and not on the number of observations $n$. When observations are sufficiently close to one another, the training set can be chosen as a subset of the observations such that $|x_i -x_j| \geq \delta \;\; \forall i,j$ for some fixed $\delta > 0$. This also avoids $\delta \approx 0$, which would lead to $k \to \infty$.  

    From Thm. \ref{Thm:SE-KErnel_pd}, we can obtain an explicit formula for the choice of $k$ for large $n$ through the following corollary. 

\begin{corollary}(Choice of $k$ for a cut-off SE-Kernel with  noise)
\label{cor:choice_k}
    Consider a set of points $x_1 < \dots < x_n$, parameters $\theta = (\sigma^2, \ell^2) \in \R^2_{++}$ and $\sns > 0$. 
    For $\delta = \min_{i\neq j} |x_i-x_j|$ choose
    \begin{equation}
    \label{eq:choice_m}
        k = \begin{cases}
        \left\lceil \sqrt{\frac{3}{2} + \frac{2 \ell^2}{\delta^2} \log \left( \frac{2 \sigma^2 \ell^2}{3 \sns \delta^2} \right)}  \right\rceil & \text{if } \frac{2 \sigma^2 \ell^2}{3 \sns \delta^2} > 1\\
        2 & \text{if not}
    \end{cases}.
    \end{equation}
Then we have that $
L_k \left( K_\theta + \sns  I \right) \succ 0.
$
\end{corollary}

\begin{proof}
    Since for $|i-j|>k$ we have that 
    \begin{equation}
         (K_\theta)_{ij} = \sigma^2 \exp \left( - \frac{(x_i-x_j)^2 }{2 \ell^2}\right) \leq \sigma^2 \exp \left( - \frac{(k \delta)^2 }{2 \ell^2}\right),
    \end{equation}
    the result follows directly from Thm. \ref{Thm:SE-KErnel_pd}.
\end{proof}

\begin{remark}
    Corollary \ref{cor:choice_k} assumes that $n$ is sufficiently large to ensure $k \leq n$.
\end{remark}

\subsection{Validity of the Predictive Distribution}
\label{sec:validity}

We denote the predictive distribution using the cut-off training covariance matrix by $\mathcal{N}(\Tilde{\boldsymbol{\mu}}, \Tilde{\Sigma} )$, where
 \begin{align}
    \Tilde{\boldsymbol{\mu}} &=  K_{\ast,\boldsymbol{f}}\left(L_k(K_{\boldsymbol{f},\boldsymbol{f}}) + \sns I \right)^{-1} \boldsymbol{y},\label{eq:pred_mu_cut}\\  \Tilde{\Sigma} &= K_{\ast,\ast} - K_{\ast,\boldsymbol{f}} \left(L_k(K_{\boldsymbol{f},\boldsymbol{f}}) + \sns I \right)^{-1}K_{\boldsymbol{f},\ast} \label{eq:pred_S_cut}
\end{align}

We seek a valid predictive distribution with a positive definite predictive covariance matrix. The following theorem states that this is the case for \MN, if $k$ is chosen as in Thm. \ref{Thm:SE-KErnel_pd}.

\begin{theorem}(Positive definiteness of the predictive covariance)
\label{Thm:pd_pred_cov}
    For an SE-Kernel and $k \in \N$ chosen as in Thm. \ref{Thm:SE-KErnel_pd}, the predictive covariance 
    \begin{equation*}
        \Tilde{\Sigma} = K_{\ast,\ast} - K_{\ast, \boldf} \left[L_k\left( K_{\boldf,\boldf} + \sns I \right)\right]^{-1}K_{\boldf,\ast}
        \label{eq:PD_predK}
    \end{equation*}
    is positive definite.
\end{theorem}

\begin{proof}
    Since from Thm. \ref{Thm:SE-KErnel_pd} we know that $L_k(K_{\boldf,\boldf} + \sns I) \succ 0$, Thm. 7.7.7. in \cite{Horn_Johnson_2012} states that positive definiteness of Eq.~\eqref{eq:PD_predK} is equivalent to  
         \begin{equation*}
         \begin{bmatrix}
        L_k(K_{\boldf,\boldf} + \sns I) & K_{\boldf, \ast}\\
        K_{\ast, \boldf} & K_{\ast,\ast}
    \end{bmatrix} \succ 0.
    \end{equation*}
    Then, the proof follows by expanding the quadratic form above and bounding each term as in the proofs of Lemma \ref{Lem:general_pd} and Thm.~\ref{Thm:SE-KErnel_pd}, and applying the Gershgorin Circle Theorem.
\end{proof}

\section{Experiments}
We evaluated BTC's performance and validated the methodology for choosing $k$ as proposed in Corollary \ref{cor:choice_k}. All experiments were run on a CPU-only Galaxy Book2 Pro Evo (Intel Core i7-1260P). All results were obtained using 5-fold cross-validation.

\subsection{Performance evaluation}

We compared BTC to the exact GP, FITC and VFE over two real-world datasets: The monthly sunspots dataset \cite{SILSO_Sunspot_Number}, and the Helsinki neonatal EEG dataset \cite{eeg}. The performance indices considered were runtime, normalised mean squared error (NMSE), given by 
\begin{equation}
\label{eq:NMSE}
\text{NMSE} = \frac{\frac{1}{n_\ast} \sum_{i=1}^{n_\ast} (y_{\ast,i} - \mu_{\ast,i})^2}{\frac{1}{n_\ast} \sum_{i=1}^{n_\ast} (y_{\ast,i} - \Bar{y}_\ast)^2}, \quad \text{where } \Bar{y}_\ast = \frac{1}{n_\ast} \sum_{i=1}^{n_\ast} y_{\ast,i},
\end{equation}
and the negative log predictive density (NLPD) of the test set
\begin{equation}
    \label{eq:NLPD}
    \text{NLPD} =  - \log p(\boldsymbol{y}_\ast | \boldsymbol{y},X,X_\ast).
\end{equation}

Our implementations of the exact GP include a \texttt{gpflow} model \cite{GPflow2017} and a basic toolbox-free implementation adapted from \cite{neurips_code}, which also forms the basis of our BTC implementation. VFE and FITC were implemented thorugh  \texttt{gpflow}.

Figs.~\ref{fig:sunspots} and \ref{fig:eeg} show performance vs runtime for the sunspots and EEG datasets, respectively. Since both NMSE and NLPD lower values indicate better performance, the best models are those close to the bottom-left corner of each plot. For BTC, results are only reported for values of $k$ that empirically preserve positive definiteness of the training covariance during optimisation.\\

For all datasets and performance indicators, BTC outperformed FITC and VFE for all choices of $m$ and $k$. For all valid bandwidths (i.e., those that preserve positive definiteness during training), BTC achieves NMSE and NLPD values that are nearly identical to those of a full GP, but at substantially reduced runtime. Furthermore, also note the almost constant performance of BTC wrt the bandwidth parameter $k$, this is a desired consequence of our approach: since $k$ was chosen according to Corollary \ref{cor:choice_k}, then it is guaranteed that the structural information of the covariance is preserved, and then no relevant information is lost. Therefore, increasing $k$ does not provide enhanced accuracy but only incorporates an additional layer of computational complexity given by the incorporation of negligible (but non-zero) entries of the covariance matrix. 

Lastly, across both datasets, our basic full GP model performed similarly to the \texttt{gpflow} implementation in terms of NMSE and NLPD, but with a higher runtime. This is expected given the multiple accelerated features of \texttt{gpflow}. Also, note that the performance of FITC is non-monotonic wrt the number of inducing variables and that in Fig.~\ref{fig:eeg} it exhibits poor NLPD \cite{FITC_VFE}.


\begin{figure}[ht]
    \centering
     \includegraphics[scale=0.295]{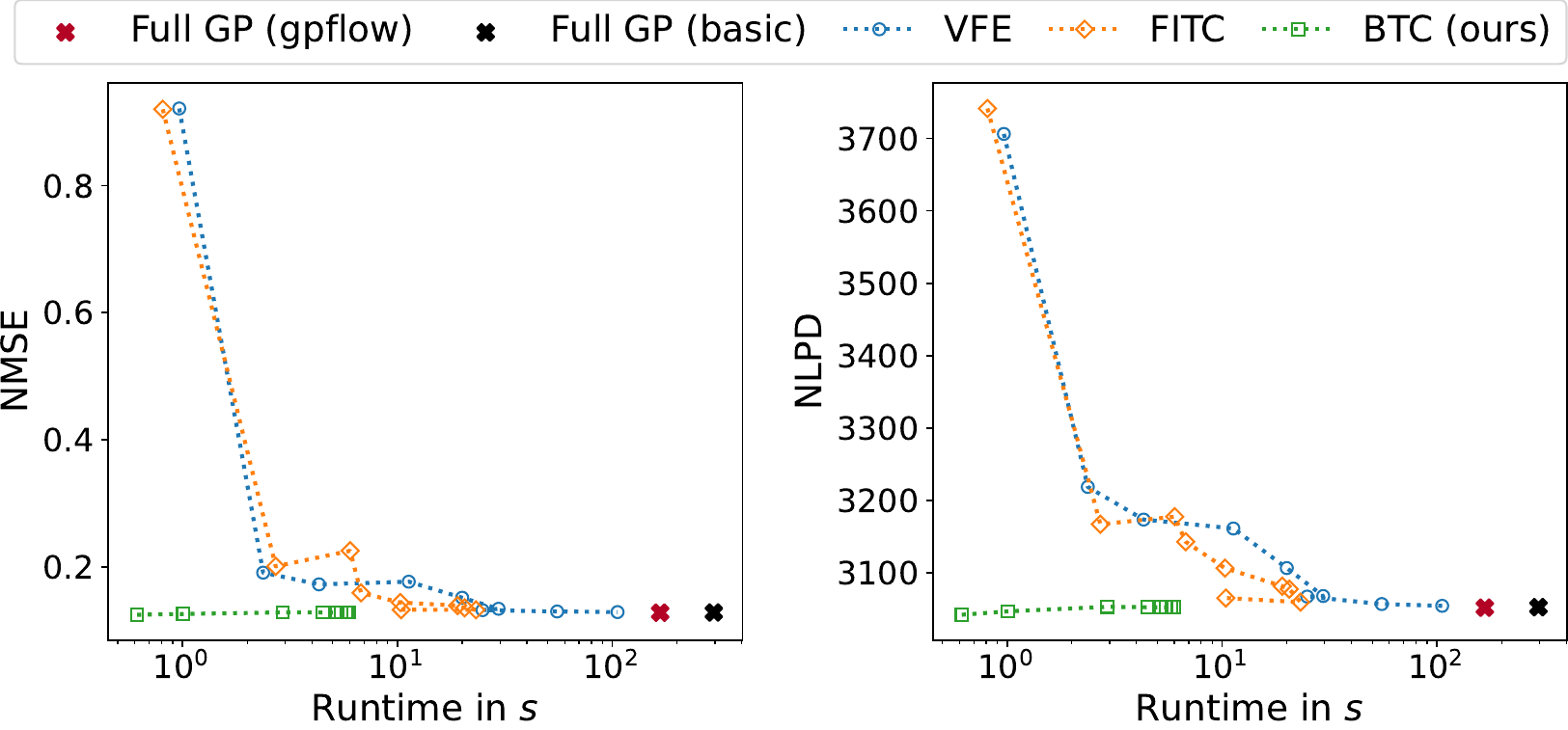}
    \caption{Sunspots dataset (3315 datapoints): NMSE (left) and NLPD (right) vs runtime. Approximation orders $m$ (FITC and VFE) and $k$ (BTC) were in the set $\{10,30,50,70,100,130,150,170,200\}$.}
    \label{fig:sunspots}
\end{figure}

\begin{figure}[ht]
    \centering
    \includegraphics[scale=0.295]{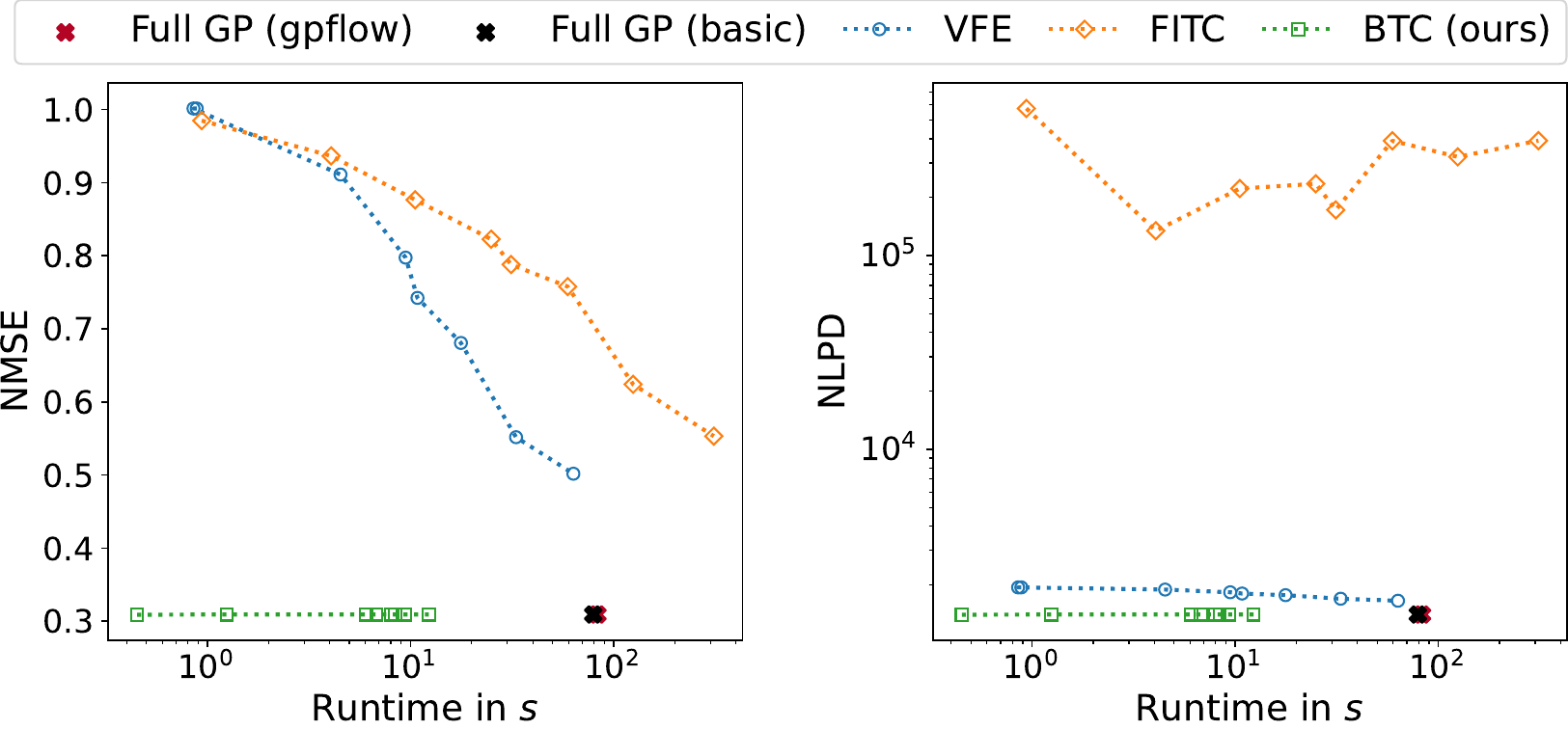}
    \caption{Neonatal EEG dataset (4000 datapoints): NMSE (left) and NLPD (right) vs runtime. Approximation orders $m$ (FITC and VFE) and $k$ (BTC) were in the set ${10,50,100,150,170,200,300,400}$.}
    \label{fig:eeg}
\end{figure}

\subsection{Validation of the theoretical choice of $k$}

Since BTC's accuracy was fairly robust wrt the bandwidth $k$ as shown in Figs.~\ref{fig:sunspots} and \ref{fig:eeg}, we then sought to determine the minimal admissible value $k$. Recall that Corollary \ref{cor:choice_k} provides a theoretically grounded choice of $k$ given the kernel (and noise) hyperparameters. To validate this theoretical result, we considered three SE-kernel GPs with different hyperparameters and computed the theoretical choice for $k$ as given by Corollary \ref{cor:choice_k} in each case; Table \ref{tab} shows the hyperparameters and chosen $K$. We then sampled one 2000-sample realisation from each GP and implemented BTC with bandwidths $k=1,\dots,50$. In this case, we considered equispaced data, since  $\delta = \min_{i \neq j} |x_i - x_j |$ often yields pessimistic choices of $k$ for non-equispaced data.

Fig.~\ref{fig:choice_k} shows BTC's prediction NMSE (green squares) for the cases where the banded covariance remained positive definite during 5-fold cross-validation, alongside the theoretical choice of $k$ (vertical red line). Observe that the choice of $k$, as given by Corollary \ref{cor:choice_k}, yielded a positive definite approximation in all cases. However, since Corollary \ref{cor:choice_k} guarantees a valid choice of $k$ for the true hyperparameters, positive definiteness throughout training is ensured only if the true parameters correspond to the highest choice of $k$ encountered during optimisation. For instance, in Fig.~\ref{fig:choice_k}(b), the true length-scale was smaller than the initial condition for that hyperparameter (set to 1), therefore, the theoretical choice of $k$ was barely sufficient. As a consequence, we recommend initialising the hyperparameters with a small length-scale and a larger noise variance than the expected true values.

\begin{figure}[ht]
    \scriptsize    \centering
    \captionof{table}{True parameters and theoretical choice of $k$ for Fig. \ref{fig:choice_k}.}
    \begin{tabular}{cccccc}
  \toprule
  \textbf{Case} & \textbf{$\delta$} & \textbf{$\sigma^2$} & \textbf{$\ell$} & \textbf{$\sigma_\text{n}^2$} & \textbf{Theo. Choice of $k$} \\
  \midrule
  (a) & 0.2 & 5   & 1    & 0.10 & 19 \\
  (b) & 0.1 & 1   & 0.75 & 0.01 & 31 \\
  (c) & 0.2 & 0.8 & 2    & 0.05 & 38 \\
  \bottomrule
\label{tab}
\end{tabular}
    \includegraphics[scale=0.295]{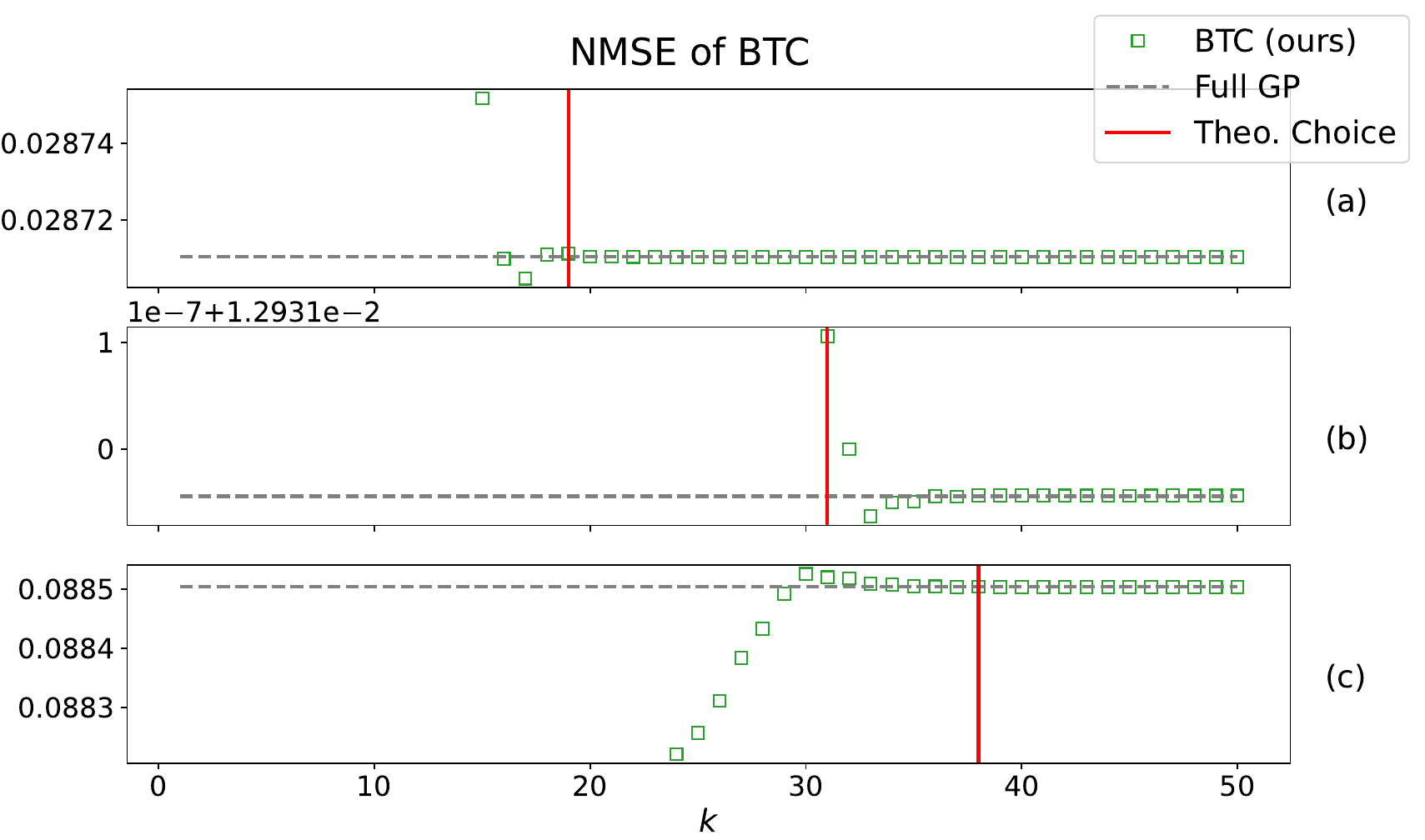}
    \captionof{figure}{Validation of the theoretical choice of $k$ on data generated with an SE kernel GP.}
    \label{fig:choice_k}
\end{figure}

\section{Conclusions}

We introduced Banded Training Covariance (BTC), a novel sparse GP method for one-dimensional GP regression with the SE kernel. BTC exploits the exponential decay of correlations to enforce a banded covariance structure, where the bandwidth $k$ that preserves positive definiteness and valid predictive distributions can be chosen independent of the sample size $n$. Our experiments showed that BTC matched the accuracy of full GPs in recovering hyperparameters, while reducing the computational cost relative to inducing-point methods. Addressing BTC’s restriction to the SE kernel and fixed bandwidth during optimisation will be key to broadening its practical applicability, particularly in higher dimensions.\\


\noindent\textbf{Acknowledgements.}
E.C.E.~was supported by a fellowship of the German Academic Exchange Service (DAAD).

\newpage

\bibliographystyle{IEEEbib}
\bibliography{references}

\end{document}